\documentclass[mlabstract]{jmlr_camera}

\usepackage{longtable}

\usepackage{booktabs}
\usepackage[load-configurations=version-1]{siunitx} 

\theorembodyfont{\upshape}
\theoremheaderfont{\scshape}
\theorempostheader{:}
\theoremsep{\newline}

\jmlrvolume{}
\firstpageno{1}
\editors{Christian Shewmake, Simone Azeglio, Bahareh Tolooshams, Sophia Sanborn, Nina Miolane}

\jmlryear{2024}
\jmlrpages{}
\jmlrworkshop{Neurips Workshop on Symmetry and Geometry in Neural Representations}

\title[Certifying Robustness via Topological Representations]{Certifying Robustness via Topological Representations}

  \author{\Name{Jens Agerberg} \Email{jensag@kth.se}\\
  \addr KTH Royal Institute of Technology
  \AND
  \Name{Andrea Guidolin} \Email{guidolin@kth.se}\\
  \addr {KTH Royal Institute of Technology \& University of Southampton}
  \AND
  \Name{Andrea Martinelli} \Email{andrea.martinelli-2@studenti.unitn.it}\\
  \addr University of Trento
  \AND
  \Name{Pepijn Roos Hoefgeest} \Email{pepijnrh@kth.se}\\
  \addr KTH Royal Institute of Technology
  \AND
  \Name{David Eklund} \Email{david.eklund@ri.se}\\
  \addr RISE Research Institutes of Sweden
  \AND
  \Name{Martina Scolamiero} \Email{scola@kth.se}\\
  \addr KTH Royal Institute of Technology
 }

\begin{document}

\maketitle

\section{Introduction}

In machine learning, the ability to obtain data representations that capture underlying geometrical and topological structures of data spaces is crucial. 
A common approach in Topological Data Analysis to extract multi-scale intrinsic geometric properties of data is persistent homology (PH) \citep{carlsson2009topology}. As a rich descriptor of  geometry,  PH has been used in machine learning pipelines in areas such as bioinformatics, neuroscience and material science \citep{dindin2020topological,colombo2022tool,lee_materialscience}. The key difference of PH compared to other methods in Geometric Deep Learning is perhaps the emphasis of theoretical stability results: PH is a Lipschitz function, with known Lipschitz constants, with respect to appropriate metrics on data and representation space \citep{cohen2005stability,skraba2020wasserstein}.

However, composing the PH pipeline with a neural network presents challenges with respect to the stability of the representations thus learned: they may lose stability or the stability may become insignificant in practice in case PH representations are composed with neural networks that have large Lipschitz constants. Moreover, the constant of the neural network may be difficult to compute or to control. While robustness to noise of PH-machine learning pipelines has been studied empirically \citep{turkevs2021noise}, we formulate the problem in the framework of adversarial learning and propose a neural network that can learn stable and discriminative geometric representations from persistence. Our contributions may be summarized as follows:


\begin{itemize}
    \item We propose the Stable Rank Network (SRN), a neural network architecture taking PH as input, where the learned representations enjoy a Lipschitz property w.r.t.\ Wasserstein and Bottleneck metrics.
    \item

     We link the stability of the PH pipeline with robustness at test time for classifiers in adversarial learning. In particular we provide certified robustness for the SRN architecture.
    
    \item On the ORBIT5K benchmark dataset, we demonstrate that the method can learn useful representations with certified robustness, when on the other hand integrating the PH pipeline with neural networks in a standard fashion can result in poor robustness properties.
\end{itemize}

\section{Background}
\subsection{Persistent homology and learnable vectorizations }
\label{subsec:PH}



In persistent homology, a filtration of a finite metric space is created by considering proximity of data points at different scales. Simplicial homology is applied to this filtration and the evolution of homology representatives to the growth of the parameter scale is tracked. These geometric patterns can be encoded in a persistence diagram (PD), a multi-set of points in the plane, on which Wasserstein metrics $W_p$ ($p \in [1, \infty]$ with $W_\infty$ being called Bottleneck) can be considered (see also Appendix \ref{appendix:dists_pers_diags}). 
Persistent homology enjoys a stability property expressed as a Lipschitz condition w.r.t.\ $W_\infty$ between PDs and Gromov-Hausdorff distance between metric spaces \citep{chazal2009gromov}.



A PD being a multi-set, Perslay \citep{carriere2020perslay} proposes to use a neural network based on the Deep Set architecture \citep{zaheer2017deep} to learn representations. Alternatively, various vectorizations \citep{Silhouettes,landscapes, Ali_2023} have been proposed, which are deterministic representations of persistence diagrams with values in a Euclidean ($L_2$) space, enabling them to be included in standard neural network pipelines.

Stable ranks \citep{scolamiero2017multidimensional, chacholski2020metrics} are learnable vectorizations with strong stability properties: while linear vectorizations of PDs to $L_2$ spaces (e.g.\ \cite{adams2017persistence}) can only be stable w.r.t.\ $W_1$ \citep{skraba2020wasserstein}, stable ranks can be designed to be 1-Lipschitz w.r.t $W_p$ for any $p \in [1, \infty]$.


\subsection{Lipschitz neural networks and robustness at test time}
\label{subsec:LipNN}

The global Lipschitz constant of a deep neural network is typically very large and may be difficult to compute and control during training. Alternative networks may however be designed to enforce a Lipschitz property, while remaining trainable and expressive. In \cite{zhang2021towards}, the authors propose to replace the classical neuron in an MLP with a unit of the form:
\begin{equation*}
    u(x, w, b) = \|x - w\|_{\infty} + b,
\end{equation*}
where $w, b \in \mathbb{R}^d$ are trainable parameters and $x \in \mathbb{R}^d$ is the input. The unit is by design 1-Lipschitz w.r.t.\ the $L_\infty$ metric on $\mathbb{R}^d$. Such neurons are stacked to form a layer and layers are composed to form a neural network enjoying the same stability property.\\

Lipschitz neural networks will be used to prove robustness at test time, in the present context referring to the property of a method of maintaining the prediction under small metric perturbations of the input. We consider robustness of trained classifiers defined on a metric space. 
 Let $(\mathcal{X}, d_{\mathcal{X}})$ be a metric space and  $g: \mathcal{X} \to \{1, \ldots, C\}$ 
 a trained classifier on $\mathcal{X}$. 
 For a sample $x \in \mathcal{X}$ with ground truth label $c$, we say that $g$ is $\epsilon$-robust at $x$ if: 
 \[
    g(x') = c, \forall x' \in \mathcal{X} \text{ s.t. } d_{\mathcal{X}}(x,x') \leq \epsilon .
\]
We are interested in estimating the \emph{robustness radius} of $x$, that is the maximum $\epsilon$ for which $g$ is $\epsilon$-robust at $x$ sampled from the data distribution. The problem of computing the robustness radius for a data sample is NP-complete even for standard MLP classifiers \citep{katz2017reluplex}.
  \textit{Adversarial attacks} are methods designed to find small perturbations of a sample that lead to a wrong classification, thus yielding an upper bound for the robustness radius at the sample. 
The field of \textit{certified robustness} \citep{li2023sok} provides instead lower bounds for the robustness radius at a data sample,  that is for a sample $x$ from the data distribution, an $\epsilon$ may be computed such that $g$ is $\epsilon$-robust at $x$.

\section {Methods}

\subsection{The Stable Rank Network and its robustness properties}
\label{SRN}
To learn a function from the space $\mathcal{PD}$ of PDs to $\mathbb{R}^C$ with a controlled Lipschitz constant we propose Stable Ranks Network (SRN), an architecture in two steps (see Figure \ref{fig:figure_pipeline}). The first step is a  vectorization using stable ranks (Section \ref{subsec:PH}), which depends on a learnable reparametrization of the filtration scale. The second step allows to learn interactions between points in the PD through a Lipschitz neural network (Section \ref{subsec:LipNN}).



The stable rank can be seen as a vectorization method for PDs which depends on parameters $(p,F)$, where $p\in[1,\infty]$ and $F: \mathbb{R}_{\geq 0} \to \mathbb{R}_{\geq 0}$ is an increasing bijection that can be interpreted as a reparameterization of the filtration scale used to apply persistent homology. For any choice of parameters $(p,F)$, with $F$ a $K$-Lipschitz function, the stable rank determines a function $r_{p,F}:\mathcal{PD} \to \mathbb{R}^\infty$ which enjoys the following Lipschitz condition:
\[
\| r_{p,F}(X) - r_{p,F}(Y)\|_\infty \le K \, W_p (X,Y),
\]
where $W_p$ denotes the $p$-Wasserstein distance between PDs  (see Appendix \ref{appendix:dists_pers_diags}).
If $F=\operatorname{id}$ is the identity on $\mathbb{R}_{\ge 0}$, the corresponding Lipschitz constant is $K=1$, for any $p\in [1,\infty]$.

In practice, one chooses an appropriate subset $\mathcal{S}$ of $\mathcal{PD}$ and a sufficiently large integer $N$ such that $r_{p,F}$ can be viewed as a function $\mathcal{S} \to \mathbb{R}^N$. In our proposed SRN architecture, $F$ is parametrized with trainable weights and $r_{p,F}$ is followed by a Lipschitz neural network with inputs in $\mathbb{R}^N$ and outputs in $\mathbb{R}^C$. Since $r_{p,F}$ is $K$-Lipschitz and all layers of the Lipschitz neural network are $1$-Lipschitz functions, their composition $f$ is a $K$-Lipschitz function from $\mathcal{S}$ equipped with $W_p$ to $\mathbb{R}^C$ equipped with the $L_{\infty}$ metric.

If we now consider a classifier based on our neural network, $g = \operatorname{argmax} \circ f$, we can (following \cite{tsuzuku2018lipschitz}) conclude that $g$ is $\epsilon$-robust at $x$ for all $\epsilon \geq \frac{M_x}{2K}$, where $K$ is the Lipschitz constant of $f$ and $M_x$ is the prediction margin of $f$ at $x$, defined as the difference between the largest logit (corresponding to the correct class) and the second largest logit: $     M_x = f(x)_c - \max_{i \neq c} f(x)_i$. Thus since the Lipschitz constant $K$ is known, computing a lower bound for the robustness radius at a sample $x$ only requires a forward pass of $x$ in the neural network to get $M_x$.

\subsection{Adversarial examples in spaces of persistence diagrams}
\label{advmethod}


Consider a trained classifier $g: \mathcal{PD} \to \{1, \ldots, C\}$ defined on the space of persistence diagrams, equipped with $W_p$ metrics $(p\in [1,\infty])$. If the Lipschitz constant of $g$ is not known, we estimate robustness of the classifier by searching for adversarial examples. The  proposed method (Appendix \ref{appendix:adv_attacks}), is derived from \cite{carlini2017towards}, by formulating the objective:
\[
    \min_{x'} W_p(x,x') - \lambda \mathcal{L}(g(x')),
\]
and iterating in the direction opposite to the gradient, i.e.\ increasing the loss $\mathcal{L}$ (e.g.\ Cross-Entropy) but keeping $W_p$ small ($\lambda$ being a hyperparameter).

\section{Results }
\label{sec:results}

We train SRN (see Appendix \ref{appendix:impl_details_sr_linf} for details) on the ORBIT5K dataset, using only PDs corresponding to degree-$1$ homology $H_1$, and we choose stability w.r.t.\ $W_\infty$ by setting $p=\infty$. In Table \ref{tab:accs_all_methods} the accuracy of the method is compared to various accuracies reported for other methods on the same problem, showing that despite the constraint of learning a Lipschitz function, competitive performance may be achieved.

We then reimplement Perslay using only $H_1$ PDs as input (achieving accuracy of $84.4 \pm 0.6$). In Table \ref{tab:rob_accs}, we report for various values of $\epsilon$, the percentage of samples in the test set such that the method described in Section \ref{advmethod} fails to produce an adversarial example within a radius of $\epsilon$ w.r.t.\ $W_\infty$. This constitutes an upper bound of the robust accuracy. While the method achieves high accuracy for the test samples, it degrades significantly even for very small $\epsilon$-neighborhoods around the samples. To get an idea of the scale we show the average intra-class and inter-class distances in Table \ref{tab:avgdisth1}. While adversarial examples are common for various types of data and neural networks, the geometry of Wasserstein metrics (e.g.\ the presence, arbitrarily close to any PD, of PDs where points are added close to the diagonal) may result in particularly poor robustness at test time.


In the case of SRN, for a sample in the test set, instead of trying to find adversarial examples, we can directly compute its certified $\epsilon$-robustness following Section \ref{SRN}. For various threshold values of $\epsilon$ we can then compute a lower bound of the robust accuracy by computing the percentage of samples in the test set for which the $\epsilon$-robustness is above this threshold (Table \ref{tab:rob_accs}). Because of this, the robust accuracies of SRN are underestimates whereas for Perslay they are overestimates. Despite this, the robust accuracies are higher for SRN and the $\epsilon$-radii appear to be meaningful compared to the average distances in Table \ref{tab:avgdisth1}.

\begin{table}[hbtp]
\floatconts
    {tab:rob_accs}
    {\caption{Robust accuracy in growing $\epsilon$-radius around samples of the test set. For Perslay we report accuracy based on robustness to the adversarial attacks described in Section \ref{advmethod}. For SRN we report accuracy based on certified robustness.}}
    {\begin{tabular}{llllll}
    \toprule
         & \bfseries Acc. & \bfseries Acc. $\epsilon=10^{-5}$ & \bfseries Acc. $\epsilon=10^{-2}$ & \bfseries Acc. $\epsilon=10^{-1}$ & \bfseries Acc. $\epsilon=10^{0}$\\
         \midrule
         Perslay ($H_1$ only) & 84.4 & 27.4 & 27.4 & 24.8 & 24.8 \\
         SRN ($H_1$ only) & 79.6 & 79.6 & 78.8 & 74.6 & 51.3 \\
         \bottomrule
    \end{tabular}}
\end{table}

\section{Conclusion}


By preserving Lipschitz properties of learned representations with respect to the Wasserstein and Bottleneck metrics between PDs, SRN allows to certify the robustness of samples in a dataset. Leveraging existing stability results of PH, a ML pipeline with interesting robustness properties w.r.t.\ appropriate distances in the input space (e.g.\ point clouds, function spaces) can thus be designed.

\newpage

\bibliography{main}

\begin{thebibliography}{31}
\providecommand{\natexlab}[1]{#1}
\providecommand{\url}[1]{\texttt{#1}}
\expandafter\ifx\csname urlstyle\endcsname\relax
  \providecommand{\doi}[1]{doi: #1}\else
  \providecommand{\doi}{doi: \begingroup \urlstyle{rm}\Url}\fi

\bibitem[Adams et~al.(2017)Adams, Emerson, Kirby, Neville, Peterson, Shipman,
  Chepushtanova, Hanson, Motta, and Ziegelmeier]{adams2017persistence}
Henry Adams, Tegan Emerson, Michael Kirby, Rachel Neville, Chris Peterson,
  Patrick Shipman, Sofya Chepushtanova, Eric Hanson, Francis Motta, and Lori
  Ziegelmeier.
\newblock Persistence images: A stable vector representation of persistent
  homology.
\newblock \emph{Journal of Machine Learning Research}, 18\penalty0
  (8):\penalty0 1--35, 2017.

\bibitem[Agerberg et~al.(2023)Agerberg, Guidolin, Ren, and
  Scolamiero]{agerberg2023algebraic}
Jens Agerberg, Andrea Guidolin, Isaac Ren, and Martina Scolamiero.
\newblock Algebraic wasserstein distances and stable homological invariants of
  data.
\newblock \emph{arXiv preprint arXiv:2301.06484}, 2023.

\bibitem[Ali et~al.(2023)Ali, Asaad, Jimenez, Nanda, Paluzo-Hidalgo, and
  Soriano-Trigueros]{Ali_2023}
Dashti Ali, Aras Asaad, Maria-Jose Jimenez, Vidit Nanda, Eduardo
  Paluzo-Hidalgo, and Manuel Soriano-Trigueros.
\newblock A survey of vectorization methods in topological data analysis.
\newblock \emph{IEEE Transactions on Pattern Analysis and Machine
  Intelligence}, 45\penalty0 (12):\penalty0 14069–14080, December 2023.
\newblock ISSN 1939-3539.
\newblock \doi{10.1109/tpami.2023.3308391}.
\newblock URL \url{http://dx.doi.org/10.1109/TPAMI.2023.3308391}.

\bibitem[Bubenik(2015)]{landscapes}
Peter Bubenik.
\newblock Statistical topological data analysis using persistence landscapes.
\newblock \emph{J. Mach. Learn. Res.}, 16\penalty0 (1):\penalty0 77–102, jan
  2015.
\newblock ISSN 1532-4435.

\bibitem[Carlini and Wagner(2017)]{carlini2017towards}
Nicholas Carlini and David Wagner.
\newblock Towards evaluating the robustness of neural networks.
\newblock In \emph{2017 ieee symposium on security and privacy (sp)}, pages
  39--57. Ieee, 2017.

\bibitem[Carlsson(2009)]{carlsson2009topology}
Gunnar Carlsson.
\newblock Topology and data.
\newblock \emph{Bulletin of the American Mathematical Society}, 46\penalty0
  (2):\penalty0 255--308, 2009.

\bibitem[Carriere et~al.(2017)Carriere, Cuturi, and Oudot]{carriere2017sliced}
Mathieu Carriere, Marco Cuturi, and Steve Oudot.
\newblock Sliced wasserstein kernel for persistence diagrams.
\newblock In \emph{International conference on machine learning}, pages
  664--673. PMLR, 2017.

\bibitem[Carri{\`e}re et~al.(2020)Carri{\`e}re, Chazal, Ike, Lacombe, Royer,
  and Umeda]{carriere2020perslay}
Mathieu Carri{\`e}re, Fr{\'e}d{\'e}ric Chazal, Yuichi Ike, Th{\'e}o Lacombe,
  Martin Royer, and Yuhei Umeda.
\newblock Perslay: A neural network layer for persistence diagrams and new
  graph topological signatures.
\newblock In \emph{International Conference on Artificial Intelligence and
  Statistics}, pages 2786--2796. PMLR, 2020.

\bibitem[Carriere et~al.(2021)Carriere, Chazal, Glisse, Ike, Kannan, and
  Umeda]{carriere2021optimizing}
Mathieu Carriere, Fr{\'e}d{\'e}ric Chazal, Marc Glisse, Yuichi Ike, Hariprasad
  Kannan, and Yuhei Umeda.
\newblock Optimizing persistent homology based functions.
\newblock In \emph{International conference on machine learning}, pages
  1294--1303. PMLR, 2021.

\bibitem[Chach{\'o}lski and Riihimaki(2020)]{chacholski2020metrics}
Wojciech Chach{\'o}lski and Henri Riihimaki.
\newblock Metrics and stabilization in one parameter persistence.
\newblock \emph{SIAM Journal on Applied Algebra and Geometry}, 4\penalty0
  (1):\penalty0 69--98, 2020.

\bibitem[Chazal et~al.(2009)Chazal, Cohen-Steiner, Guibas, M{\'e}moli, and
  Oudot]{chazal2009gromov}
Fr{\'e}d{\'e}ric Chazal, David Cohen-Steiner, Leonidas~J Guibas, Facundo
  M{\'e}moli, and Steve~Y Oudot.
\newblock Gromov-hausdorff stable signatures for shapes using persistence.
\newblock In \emph{Computer Graphics Forum}, volume~28, pages 1393--1403. Wiley
  Online Library, 2009.

\bibitem[Chazal et~al.(2014)Chazal, Fasy, Lecci, Rinaldo, and
  Wasserman]{Silhouettes}
Fr\'{e}d\'{e}ric Chazal, Brittany~Terese Fasy, Fabrizio Lecci, Alessandro
  Rinaldo, and Larry Wasserman.
\newblock Stochastic convergence of persistence landscapes and silhouettes.
\newblock In \emph{Proceedings of the Thirtieth Annual Symposium on
  Computational Geometry}, SOCG'14, page 474–483, New York, NY, USA, 2014.
  Association for Computing Machinery.
\newblock ISBN 9781450325943.
\newblock \doi{10.1145/2582112.2582128}.
\newblock URL \url{https://doi.org/10.1145/2582112.2582128}.

\bibitem[Cohen-Steiner et~al.(2005)Cohen-Steiner, Edelsbrunner, and
  Harer]{cohen2005stability}
David Cohen-Steiner, Herbert Edelsbrunner, and John Harer.
\newblock Stability of persistence diagrams.
\newblock In \emph{Proceedings of the twenty-first annual symposium on
  Computational geometry}, pages 263--271, 2005.

\bibitem[Cohen-Steiner et~al.(2010)Cohen-Steiner, Edelsbrunner, Harer, and
  Mileyko]{cohen2010lipschitz}
David Cohen-Steiner, Herbert Edelsbrunner, John Harer, and Yuriy Mileyko.
\newblock Lipschitz functions have l p-stable persistence.
\newblock \emph{Foundations of computational mathematics}, 10\penalty0
  (2):\penalty0 127--139, 2010.

\bibitem[Colombo et~al.(2022)Colombo, Cubero, Kanari, Venturino, Schulz,
  Scolamiero, Agerberg, Mathys, Tsai, Chach{\'o}lski, et~al.]{colombo2022tool}
Gloria Colombo, Ryan John~A Cubero, Lida Kanari, Alessandro Venturino, Rouven
  Schulz, Martina Scolamiero, Jens Agerberg, Hansruedi Mathys, Li-Huei Tsai,
  Wojciech Chach{\'o}lski, et~al.
\newblock A tool for mapping microglial morphology, morphomics, reveals
  brain-region and sex-dependent phenotypes.
\newblock \emph{Nature neuroscience}, 25\penalty0 (10):\penalty0 1379--1393,
  2022.

\bibitem[Dindin et~al.(2020)Dindin, Umeda, and Chazal]{dindin2020topological}
Meryll Dindin, Yuhei Umeda, and Frederic Chazal.
\newblock Topological data analysis for arrhythmia detection through modular
  neural networks.
\newblock In \emph{Advances in Artificial Intelligence: 33rd Canadian
  Conference on Artificial Intelligence, Canadian AI 2020, Ottawa, ON, Canada,
  May 13--15, 2020, Proceedings 33}, pages 177--188. Springer, 2020.

\bibitem[Edelsbrunner(1995)]{edelsbrunner1995union}
Herbert Edelsbrunner.
\newblock The union of balls and its dual shape.
\newblock \emph{Discrete Comput Geom}, 13:\penalty0 415--440, 1995.

\bibitem[Katz et~al.(2017)Katz, Barrett, Dill, Julian, and
  Kochenderfer]{katz2017reluplex}
Guy Katz, Clark Barrett, David~L Dill, Kyle Julian, and Mykel~J Kochenderfer.
\newblock Reluplex: An efficient smt solver for verifying deep neural networks.
\newblock In \emph{Computer Aided Verification: 29th International Conference,
  CAV 2017, Heidelberg, Germany, July 24-28, 2017, Proceedings, Part I 30},
  pages 97--117. Springer, 2017.

\bibitem[Kusano et~al.(2016)Kusano, Hiraoka, and
  Fukumizu]{kusano2016persistence}
Genki Kusano, Yasuaki Hiraoka, and Kenji Fukumizu.
\newblock Persistence weighted gaussian kernel for topological data analysis.
\newblock In \emph{International conference on machine learning}, pages
  2004--2013. PMLR, 2016.

\bibitem[Le and Yamada(2018)]{le2018persistence}
Tam Le and Makoto Yamada.
\newblock Persistence fisher kernel: A riemannian manifold kernel for
  persistence diagrams.
\newblock \emph{Advances in neural information processing systems}, 31, 2018.

\bibitem[Lee et~al.(2017)Lee, Barthel, D{\l}otko, Moosavi, Hess, and
  Smit]{lee_materialscience}
Yongjin Lee, Senja~D. Barthel, Pawe{\l} D{\l}otko, S.~Mohamad Moosavi, Kathryn
  Hess, and Berend Smit.
\newblock Quantifying similarity of pore-geometry in nanoporous materials.
\newblock \emph{Nature Communications}, 8\penalty0 (1):\penalty0 15396, 2017.
\newblock \doi{10.1038/ncomms15396}.
\newblock URL \url{https://doi.org/10.1038/ncomms15396}.

\bibitem[Li et~al.(2023)Li, Xie, and Li]{li2023sok}
Linyi Li, Tao Xie, and Bo~Li.
\newblock Sok: Certified robustness for deep neural networks.
\newblock In \emph{2023 IEEE symposium on security and privacy (SP)}, pages
  1289--1310. IEEE, 2023.

\bibitem[Otter et~al.(2017)Otter, Porter, Tillmann, Grindrod, and
  Harrington]{otter2017roadmap}
Nina Otter, Mason~A Porter, Ulrike Tillmann, Peter Grindrod, and Heather~A
  Harrington.
\newblock A roadmap for the computation of persistent homology.
\newblock \emph{EPJ Data Science}, 6:\penalty0 1--38, 2017.

\bibitem[Oudot(2017)]{oudot2017persistence}
Steve~Y Oudot.
\newblock \emph{Persistence theory: from quiver representations to data
  analysis}, volume 209.
\newblock American Mathematical Soc., 2017.

\bibitem[Reininghaus et~al.(2015)Reininghaus, Huber, Bauer, and
  Kwitt]{reininghaus2015stable}
Jan Reininghaus, Stefan Huber, Ulrich Bauer, and Roland Kwitt.
\newblock A stable multi-scale kernel for topological machine learning.
\newblock In \emph{Proceedings of the IEEE conference on computer vision and
  pattern recognition}, pages 4741--4748, 2015.

\bibitem[Scolamiero et~al.(2017)Scolamiero, Chach{\'o}lski, Lundman, Ramanujam,
  and {\"O}berg]{scolamiero2017multidimensional}
Martina Scolamiero, Wojciech Chach{\'o}lski, Anders Lundman, Ryan Ramanujam,
  and Sebastian {\"O}berg.
\newblock Multidimensional persistence and noise.
\newblock \emph{Foundations of Computational Mathematics}, 17:\penalty0
  1367--1406, 2017.

\bibitem[Skraba and Turner(2020)]{skraba2020wasserstein}
Primoz Skraba and Katharine Turner.
\newblock Wasserstein stability for persistence diagrams.
\newblock \emph{arXiv preprint arXiv:2006.16824}, 2020.

\bibitem[Tsuzuku et~al.(2018)Tsuzuku, Sato, and Sugiyama]{tsuzuku2018lipschitz}
Yusuke Tsuzuku, Issei Sato, and Masashi Sugiyama.
\newblock Lipschitz-margin training: Scalable certification of perturbation
  invariance for deep neural networks.
\newblock \emph{Advances in neural information processing systems}, 31, 2018.

\bibitem[Turke{\v{s}} et~al.(2021)Turke{\v{s}}, Nys, Verdonck, and
  Latr{\'e}]{turkevs2021noise}
Renata Turke{\v{s}}, Jannes Nys, Tim Verdonck, and Steven Latr{\'e}.
\newblock Noise robustness of persistent homology on greyscale images, across
  filtrations and signatures.
\newblock \emph{Plos one}, 16\penalty0 (9):\penalty0 e0257215, 2021.

\bibitem[Zaheer et~al.(2017)Zaheer, Kottur, Ravanbakhsh, Poczos, Salakhutdinov,
  and Smola]{zaheer2017deep}
Manzil Zaheer, Satwik Kottur, Siamak Ravanbakhsh, Barnabas Poczos, Russ~R
  Salakhutdinov, and Alexander~J Smola.
\newblock Deep sets.
\newblock \emph{Advances in neural information processing systems}, 30, 2017.

\bibitem[Zhang et~al.(2021)Zhang, Cai, Lu, He, and Wang]{zhang2021towards}
Bohang Zhang, Tianle Cai, Zhou Lu, Di~He, and Liwei Wang.
\newblock Towards certifying l-infinity robustness using neural networks with
  l-inf-dist neurons.
\newblock In \emph{International Conference on Machine Learning}, pages
  12368--12379. PMLR, 2021.

\end{thebibliography}

\newpage
\appendix



\section{Persistent homology and distances}

\subsection{Persistence diagrams}
\label{subsec:PDs}



Consider the subset $U := \{(a,b)\in \mathbb{R}_{\ge 0}\times [0,\infty ] \mid a\leq b\}$ of the extended plane. 
A \emph{persistence diagram} $D$ is a multiset of elements of $U$, which in this work we assume to be finite. By definition of multiset, $D$ is a pair $(X,\mu)$, where $X$ is a finite subset of $U$ and $\mu : X \to \mathbb{N}_{>0}$ is a function, interpreted as a multiplicity function. The cardinality of $D$, denoted by $\operatorname{rank}(D)$, is $\sum_{x\in X}\mu (x)$. A persistence diagram $D$ can be viewed as the set $\{ (x,i_x)\in X\times \mathbb{N}_{>0} \mid 1\le i_x \le \mu (x) \}$. In what follows, for simplicity we disregard the second components of the elements $(x,i_x)$ of $D$, using the notation $D=\{x_i\}_{i = 1, \ldots ,\operatorname{rank}(D)}$, with $x_i=(a_i,b_i)\in U$, where it may happen that $x_j=x_k$ for some $j\ne k$. We use the notation $\mathcal{PD}$ for the collection of all persistence diagrams.


In topological data analysis \citep{carlsson2009topology}, persistence diagrams are usually obtained by applying $q$-th simplicial homology $H_q$ (for some $q\in \mathbb{N}$) to a nested sequence $K_0\subseteq K_1 \subseteq \cdots$ of simplicial complexes, called a filtration, associated with the data. Details on the pipeline of persistent homology and on various constructions to transform the data into a filtration of simplicial complexes are included for example in \cite{otter2017roadmap,oudot2017persistence}. In this work, the data is in the form of point clouds, which are transformed into filtrations of simplicial complexes using the alpha-complex construction \citep{edelsbrunner1995union}.  


An element $x=(a,b)$ of a persistence diagram is called a \emph{point at infinity} if $b=\infty$. In this work, all persistence diagrams we consider do not have points at infinity. Starting from point cloud data and using the Alpha complex construction and simplicial homology $H_q$, one always obtains persistence diagrams without points at infinity for every $q>0$.


\subsection{Wasserstein distance between persistence diagrams} 
\label{appendix:dists_pers_diags}

Let $U$ be the subset of the extended plane defined in Section \ref{subsec:PDs}, and consider its subset $\Delta :=  \{ (a,a) \mid a \in \mathbb{R}_{\ge 0} \}$.  For all $p\in [1,\infty]$, let $d_p (x,y) :=  \left\| x-y \right\|_p$ for all $x,y\in U$, and let $d_p (x,\Delta) :=  \inf_{z\in \Delta} d_p (x,z)$, for all $x\in U$. 

For $p\in [1,\infty]$, the $p$-\emph{Wasserstein distance} between two persistence diagrams $D=\{x_i\}_{i=1,\ldots, m}$ and $D'=\{x'_j\}_{j=1,\ldots, n}$ is defined by
\begin{multline*}
W_p (D,D') :=  \\
\inf_{\alpha} \left\| \left( \left\| (d_p (x_i,x'_{\alpha (i)}))_{i\in I} \right\|_p , \left\| (d_p (x_i,\Delta))_{i\in \{1,\ldots ,m\}\setminus I} \right\|_p , \left\| (d_p (\Delta, x'_j))_{j\in \{1,\ldots ,n\}\setminus \alpha (I)} \right\|_p  \right) \right\|_p ,
\end{multline*}
where the infimum is over all injective functions $\alpha \colon  I \to \{1,\ldots ,n\}$, with $I\subseteq \{1,\ldots ,m\}$.

The \emph{Bottleneck distance}, which is of special importance in the literature on persistent homology, is the distance $W_{\infty}$ corresponding to the case $p=\infty$. We remark that the definition of $p$-Wasserstein distance we use here differs from that of some other authors, which use the metric $d_{\infty}$ on $U$, for all $p\in [1,\infty]$ (see e.g.\ \cite{cohen2010lipschitz}).







Persistence diagrams enjoy crucial stability properties. 
For example, consider a filtration of simplicial complexes constructed as a sublevel-set filtration of a simplicial complex equipped with a function on its set of simplices. Under mild assumptions on this function, the Bottleneck distance on persistence diagrams is stable w.r.t.\ perturbations of the function in $L_{\infty}$ norm \citep{cohen2005stability}. Stability results also exist for other types of data from which a filtration of simplicial complexes may be constructed, e.g.\ Gromov-Hausdorff distance between point clouds. An account of stability results involving the $p$-Wasserstein distance is given in \citep{skraba2020wasserstein}.

\section{Stable ranks}

\paragraph{Definition of stable ranks}
In this section, all persistence diagrams are assumed to have no points at infinity.
Let $V := \{(a,b)\in \mathbb{R}_{\ge 0}^2 \mid a\leq b\}$ be the subset of $U$ (see Section \ref{subsec:PDs}) obtained by removing all points at infinity. Let $F:\mathbb{R}_{\ge 0} \to \mathbb{R}_{\ge 0}$ be an increasing bijection, which we call a \emph{reparameterization} of $\mathbb{R}_{\ge 0}$. Parametrized families of increasing bijections can be produced from regular contours introduced in \cite{chacholski2020metrics}. 
Given a persistence diagram $D\in \mathcal{PD}$, let $F(D)$ denote the persistence diagram obtained by transforming each point $x=(a,b)$ of $D$ into the point $F(x):=(F(a),F(b))$.
The \emph{stable rank} of a persistence diagram $D$ with parameters $(p,F)$, where $p\in [1,\infty]$ and $F$ is a reparameterization of $\mathbb{R}_{\ge 0}$, is the function $\widehat{\operatorname{rank}}_{p,F}(D): \mathbb{R}_{\ge 0} \to \mathbb{N}$ defined by
\[
\widehat{\operatorname{rank}}_{p,F}(D)(t) := \min \{\operatorname{rank}(D') \mid D'\in \mathcal{PD} \text{ and } W_p(F(D),F(D'))\le t \},
\]
where $W_p$ denotes the $p$-Wasserstein distance between persistence diagrams and $\operatorname{rank}(D')$ denotes the cardinality of the persistence diagram $D'$ (or, equivalently, of $F(D')$). The function $\widehat{\operatorname{rank}}_{p,F}(D)$ is nonincreasing and piecewise constant. One way to compare such functions is via the \emph{interleaving distance} $d_{\bowtie}$, see \cite[Def.~9.1]{scolamiero2017multidimensional}. Stable ranks were originally defined for persistence modules (see \cite{scolamiero2017multidimensional, chacholski2020metrics}), algebraic objects which are completely described (up to isomorphism) by the associated persistence diagrams. The stable ranks we introduce here for persistence diagrams can be regarded as an instance of the original definition thanks to the equivalence of Wasserstein distances between persistence modules and persistence diagrams detailed in \cite[Sect.~4.4]{agerberg2023algebraic}. 

\paragraph{Computation of stable ranks}
Let $\ell_F : V \to \mathbb{R}_{\ge 0}$ denote the function defined by $\ell_F(a,b):= F(b)-F(a)$, which we call the \emph{lifetime function} induced by the reparameterization $F$.
Let $D=\{x_i\}_{i=1,\ldots, m}$ be a persistence diagram such that $x_i =(a_i,b_i)\in V$, for all $i\in \{1,\ldots ,m\}$, and suppose that its points are ordered non-decreasingly by their lifetime $\ell_F$, meaning that $\ell_F (a_1,b_1)\le \ell_F(a_2,b_2)\le \cdots \le \ell_F(a_m,b_m)$. 
Then, by \cite[Prop.~5.1]{agerberg2023algebraic}, for any fixed $p\in [1,\infty)$ there exist real numbers $0=t_0 < t_1 < t_2 < \cdots < t_m$ such that the function $\widehat{\operatorname{rank}}_{p,F}(D) \colon  \mathbb{R}_{\ge 0} \to \mathbb{N}$ is constant on the intervals $[t_0,t_1)$, $[t_1,t_2)$,\ldots , $[t_{m-1},t_m)$, $[t_m,\infty)$, with values
$\widehat{\operatorname{rank}}_{p,F}(D)(t_j) = \operatorname{rank} (D) - j$,
for every $j\in \{0,1,\ldots ,m\}$. Furthermore,
\begin{equation}
\label{eq:tj}
t_j= 2^{\frac{1-p}{p}}\left\| ( \ell_F(a_1,b_1),\ldots ,\ell_F(a_j,b_j) )\right\|_p 
\end{equation}
for every $j\in \{1,\ldots ,m\}$.

For $p=\infty$, the sequence of real numbers $(t_j)_j$ defined in Equation (\ref{eq:tj}) (setting $2^{\frac{1-p}{p}}=2^{-1}$) only satisfies $0=t_0 \le t_1 \le t_2 \le \cdots \le t_m$ instead of strict inequalities. Letting $s_k$ denote the $k^{\text{th}}$ smallest value in $\{ t_j\}_j$ one obtains a sequence $0=s_0 < s_1 < s_2 < \cdots < s_{m'}$ such that the stable rank with parameters $(p=\infty,F)$ is constant on the intervals $[s_0,s_1)$,\ldots , $[s_{m'},\infty)$, with values 
\[
\widehat{\operatorname{rank}}_{\infty, F}(D)(s_k) = \operatorname{rank} (D) - \max \{j \mid t_j=s_k \} .
\]

For any fixed $p\in [1,\infty]$, the function $\widehat{\operatorname{rank}}_{p,F}(D)$ can be represented as a vector $r_{p,F}(D) = (r_0, r_1, \ldots ,r_m) \in \mathbb{R}^{m+1}$, where
\[
r_i := t_{m-i} = 2^{\frac{1-p}{p}} \| (\ell_F(a_1,b_1) , \ldots ,\ell_F(a_{m-i},b_{m-i})) \|_p, 
\]
for all $i\in \{0,\ldots ,m-1\}$, and $r_m := 0$. We remark that, thanks to the assumption that the persistence diagram $D$ has no points at infinity, all $t_j$ and all $r_i$ are real numbers. By construction (see \cite[Sect.~5]{agerberg2023algebraic}), the entries of $r_{p,F}(D)$ satisfy $r_i = \min \{t \in \mathbb{R}_{\ge 0} \mid \widehat{\operatorname{rank}}_{p,F}(D) \le i\}$, for all $i\in \{0,\ldots ,m\}$. 

\paragraph{Computing distances between stable ranks}
Since $m=\operatorname{rank}(D)$ depends on the persistence diagram $D$, to compare the vectors $r_{p,F}(D)\in \mathbb{R}^{m+1}$ associated with different persistence diagrams we regard them as elements of the vector space $\mathbb{R}^\infty$ of finite sequences over $\mathbb{R}$, defined by
\[
\mathbb{R}^\infty := \{ (r_0,r_1, \ldots) \in \mathbb{R}^{\mathbb{N}}\mid \text{finitely many entries $r_i$ are nonzero}\}, 
\]
with the natural vector space structure given by entry-wise addition and scalar multiplication. Any vector $r_{p,F}(D)=(r_0,r_1,\ldots ,r_m)\in \mathbb{R}^{m+1}$ is regarded as the element $(r_0,r_1,\ldots ,r_m, 0, \ldots)$ of $\mathbb{R}^{\infty}$, which we will also denote by $r_{p,F}$ with a small abuse of notation. 

Importantly, as a consequence of the results in \cite[Sect.~5]{agerberg2023algebraic}, the interleaving distance between stable ranks can be computed via the $L_{\infty}$ metric in $\mathbb{R}^{\infty}$:
\begin{equation}
\label{eq:interleavingLinfty}
d_{\bowtie}(\widehat{\operatorname{rank}}_{p,F}(D),\widehat{\operatorname{rank}}_{p,F}(D')) = \| r_{p,F}(D) - r_{p,F}(D')\|_{\infty},
\end{equation}
for any $D,D'\in \mathcal{PD}$.

In practice, when using stable ranks in data analysis we work on a subset $\mathcal{S}\subset \mathcal{PD}$ of persistence diagrams which allows us to compare the vectors $\{r_{p,F}(D)\}_{D\in \mathcal{S}}$ via Equation (\ref{eq:interleavingLinfty}) in a finite dimensional space $\mathbb{R}^{d}$ (for a sufficiently large $d$) instead of $\mathbb{R}^{\infty}$. For example, in many practical situations one can consider $\mathcal{S}=\{D\in \mathcal{PD} \mid \operatorname{rank}(D) \le d-1\}$, for a fixed and sufficiently large $d$. 

\paragraph{Stability of stable ranks}
\label{appendix:stab_rank_stability}
Combining the equality in (\ref{eq:interleavingLinfty}) with a stability result for stable ranks \cite[Prop.~9.3]{scolamiero2017multidimensional}, we obtain the inequality
\begin{equation}
\label{eq:1Lipsch_rpF}
\| r_{p,F}(D) - r_{p,F}(D')\|_{\infty} \le W_p(F(D),F(D')),
\end{equation}
valid for any $D,D'\in \mathcal{PD}$ and for any fixed $p\in [1,\infty]$ and reparameterization $F$. In particular, if $F=\operatorname{id}$ is the identity reparameterization of $\mathbb{R}_{\ge 0}$, we see that $r_{p,\operatorname{id}}$ is a $1$-Lipschitz function from $\mathcal{PD}$ equipped with the $p$-Wasserstein distance $W_p$ to $\mathbb{R}^{\infty}$ equipped with the $L_{\infty}$ metric:
\begin{equation}
\label{eq:1Lipsch_rp}
\| r_{p,\operatorname{id}}(D) - r_{p,\operatorname{id}}(D')\|_{\infty} \le W_p(D,D'),
\end{equation}
for all $D,D'\in \mathcal{PD}$.

In the following results we study the Lipschitzianity of the function $r_{p,F}$ with respect to the $p$-Wasserstein distance on $\mathcal{PD}$.

\begin{proposition}
\label{prop:Lipschitz_const}
Let $p\in [1,\infty]$ and let $F:\mathbb{R}_{\ge 0}\to \mathbb{R}_{\ge 0}$ be a reparameterization which is $K$-Lipschitz for some $K>0$, that is, $\lvert F(b) - F(a) \rvert \le K \lvert b-a\rvert $ for all $a,b\in \mathbb{R}_{\ge 0}$. Then
\[
W_p(F(D),F(D')) \le K \, W_p(D,D'),
\]
for all $D,D'\in \mathcal{PD}$. 
\end{proposition}
\begin{proof}
Let $D=\{x_i\}_{i=1,\ldots, m}$ and $D'=\{x'_j\}_{j=1,\ldots, n}$ be two persistence diagrams, with $x_i =(a_i,b_i)\in V$ for all $i\in \{1,\ldots ,m\}$ and $x'_j =(a'_j,b'_j)\in V$ for all $j\in \{1,\ldots ,n\}$.  
For any injective function $\alpha :I\to \{1,\ldots ,n\}$ defined on a subset $I\subseteq \{1,\ldots ,m\}$, we consider the real numbers:
\begin{align*}
u_{p,F,\alpha} &:= \left\| (d_p (F(x_i),F(x'_{\alpha (i))}))_{i\in I} \right\|_p , \\   
v_{p,F,\alpha} &:= \left\| (d_p (F(x_i),\Delta))_{i\in \{1,\ldots ,m\}\setminus I} \right\|_p , \\
w_{p,F,\alpha} &:= \left\| (d_p (\Delta, F(x'_j)))_{j\in \{1,\ldots ,n\}\setminus \alpha (I)} \right\|_p .
\end{align*}
By definition (see Section \ref{appendix:dists_pers_diags}) we have $W_p(F(D),F(D'))= \inf_{\alpha} \| (u_{p,F,\alpha}, v_{p,F,\alpha}, w_{p,F,\alpha}) \|_p$ and $W_p(D,D')= \inf_{\alpha} \| (u_{p,\operatorname{id},\alpha}, v_{p,\operatorname{id},\alpha}, w_{p,\operatorname{id},\alpha}) \|_p$. Recall the following property of the $p$-norm: if $0\le u_k \le u'_k$ for $k\in \{1,\ldots ,\ell\}$, then $\| (u_1, \ldots ,u_\ell) \|_p \le \| (u'_1, \ldots ,u'_\ell) \|_p$. To prove the claim, it is therefore sufficient to show that, for any fixed $\alpha$, the following inequalities hold:
\[
u_{p,F,\alpha} \le K u_{p,\operatorname{id},\alpha}, \quad v_{p,F,\alpha} \le K v_{p,\operatorname{id},\alpha}, \quad w_{p,F,\alpha} \le K w_{p,\operatorname{id},\alpha},
\]
where $K$ is the Lipschitz constant of $F$. 

To prove the inequality $u_{p,F,\alpha} \le K u_{p,\operatorname{id},\alpha}$ it is sufficient to show that, for any fixed $i\in I$, we have 
\[
d_p (F(x_i),F(x'_{\alpha (i)})) \le K d_p (x_i,x'_{\alpha (i)}).
\]
This follows from the calculation:
\begin{align*}
d_p (F(x_i),F(x'_{\alpha (i)})) &= \left\| \left(F(a_i) - F(a'_{\alpha (i)}), F(b_i) - F(b'_{\alpha (i)})\right) \right\|_p \\
&= \left\| \left( \lvert F(a_i) - F(a'_{\alpha (i)})\rvert , \lvert F(b_i) - F(b'_{\alpha (i)}) \rvert \right) \right\|_p \\
&\le \left\| \left( K \lvert a_i - a'_{\alpha (i)}\rvert , K \lvert b_i - b'_{\alpha (i)} \rvert \right) \right\|_p \\
&= K \left\| \left( a_i - a'_{\alpha (i)}, b_i - b'_{\alpha (i)} \right) \right\|_p \\
&=  K  d_p (x_i,x'_{\alpha (i)}).
\end{align*}

To prove the inequality $v_{p,F,\alpha} \le K v_{p,\operatorname{id},\alpha}$ it is sufficient to show that, for any fixed $i\in \{1,\ldots ,m\} \setminus I$, we have 
\[
d_p (F(x_i),\Delta) \le K d_p (x_i,\Delta).
\]
For any $x=(a,b)\in V$, it is easy to show that $d_p (x,\Delta)= d_p (x,\overline{z})$ with $\overline{z}:= (\frac{a+b}{2},\frac{a+b}{2})$.
The inequality then follows from the calculation:
\begin{align*}
d_p (F(x_i),\Delta ) &= \left\| \left(F(a_i) - \frac{F(a_i)+F(b_i)}{2}, F(b_i) - \frac{F(a_i)+F(b_i)}{2}\right) \right\|_p \\
&= \left\| \left( \frac{F(a_i)-F(b_i)}{2}, \frac{F(b_i)-F(a_i)}{2}\right) \right\|_p \\
&= \left\| \left( \frac{F(b_i)-F(a_i)}{2}, \frac{F(b_i)-F(a_i)}{2}\right) \right\|_p \\
&\le \left\| \left( \frac{K(b_i-a_i)}{2}, \frac{K(b_i-a_i)}{2}\right) \right\|_p \\
&= K \left\| \left( \frac{(b_i-a_i)}{2}, \frac{(b_i-a_i)}{2}\right) \right\|_p \\
&= K d_p (x_i,\Delta ).
\end{align*}

The inequality $w_{p,F,\alpha} \le K w_{p,\operatorname{id},\alpha}$ can be proven similarly. This completes the proof.
\end{proof}

\begin{corollary}
\label{coro:Lipschitz_const}    
Let $p\in [1,\infty]$ and let $F:\mathbb{R}_{\ge 0}\to \mathbb{R}_{\ge 0}$ be a reparameterization which is continuous and differentiable, with $F'$ bounded. Then
\[
W_p(F(D),F(D')) \le \sup F' \, W_p(D,D'),
\]
for all $D,D'\in \mathcal{PD}$. 
\end{corollary}
\begin{proof}
Given any two persistence diagrams $D=\{x_i\}_{i=1,\ldots, m}$ and $D'=\{x'_j\}_{j=1,\ldots, n}$, with $x_i =(a_i,b_i)\in V$ for all $i\in \{1,\ldots ,m\}$ and $x'_j =(a'_j,b'_j)\in V$ for all $j\in \{1,\ldots ,n\}$, consider a closed interval $[c,d]\subset \mathbb{R}$ containing all $a_i,b_i,a'_j,b'_j$.  
By the Mean Value Theorem, the restriction of $F$ to $[c,d]$ is $K$-Lipschitz for $K=\sup F'$. Proceeding as in Proposition \ref{prop:Lipschitz_const}, we obtain the inequality $W_p(F(D),F(D')) \le \sup F' \, W_p(D,D')$. Since the Lipschitz constant $\sup F'$ does not depend on the persistence diagrams $D$ and $D'$, the claim follows.
\end{proof}

Equation (\ref{eq:1Lipsch_rpF}) combined with Proposition \ref{prop:Lipschitz_const} shows that, whenever $F$ is $K$-Lipschitz, the function $r_{p,F}$ from $\mathcal{PD}$ equipped with the $p$-Wasserstein distance $W_p$ to $\mathbb{R}^{\infty}$ equipped with the $L_{\infty}$ metric is $K$-Lipschitz: 
\begin{equation*}
\| r_{p,F}(D) - r_{p,F}(D')\|_{\infty} \le K\, W_p(D,D'),
\end{equation*}
for all $D,D'\in \mathcal{PD}$.
In particular, by Corollary \ref{coro:Lipschitz_const}, $r_{p,F}$ is $(\sup F')$-Lipschitz whenever $F$ is differentiable with $\sup F'<\infty$.
In practice, one can construct a differentiable parameterization $F$ with bounded derivative by choosing a bounded integrable function $f:\mathbb{R}_{\ge 0} \to \mathbb{R}_{> 0}$ and setting $F(t) =\int_0^t f\, d\lambda$, for all $t\in \mathbb{R}_{\ge 0}$.

\section{Experiments}

\subsection{ORBITS5K dataset}
\label{subsec:ORBdataset}

Starting from an initial random position $(x_0, y_0) \in [0, 1]^2$ and a parameter $r > 0$, the point cloud $(x_n, y_n)_{n=1,...,1000}$ is generated using the following system:

\[
x_{n+1} = (x_n + r y_n(1 - y_n)) \mod 1
\]
\[
y_{n+1} = (y_n + r x_{n+1}(1 - x_{n+1})) \mod 1
\]

The orbits' behavior is highly sensitive to $r$, with some values leading to void formations, making persistence diagrams effective for classifying these orbits by $r$. Following \cite{carriere2020perslay}, we simulate orbits using five $r$ values: 2.5, 3.5, 4.0, 4.1, and 4.3. 1000 orbits are generated for each $r$, resulting in 5,000 point clouds.

The point clouds are transformed into persistence diagrams using the Alpha complex filtration and simplicial homology in degree 0 and 1. The Alpha complex filtration is a method to generate a filtered simplicial complex, starting from a point cloud \( P \) in \( \mathbb{R}^d \). The Alpha complex \( \mathcal{A}_\alpha \) at a scale parameter \( \alpha \) is constructed by first considering the union of balls centered at the points in \( P \) with radius \( \alpha \). The simplices in \( \mathcal{A}_\alpha \) correspond to the simplices in the Delaunay triangulation of \( P \) that are enclosed within the union of these balls. As \( \alpha \) increases, the Alpha complex grows, generating the filtration \( \{ \mathcal{A}_\alpha \}_{\alpha \geq 0} \).

The coordinates of the persistence diagrams are scaled by $1000$ to make the reporting of Wasserstein distances more convenient (e.g.\ in Table \ref{tab:avgdisth1}). See Figure \ref{fig:image_orbits5k} for examples of samples in the dataset.

\begin{figure}[htbp]
\floatconts
  {fig:image_orbits5k}
  {\caption{An example sample for each class in the ORBIT5K dataset.}}
  {\includegraphics[width=1\linewidth]{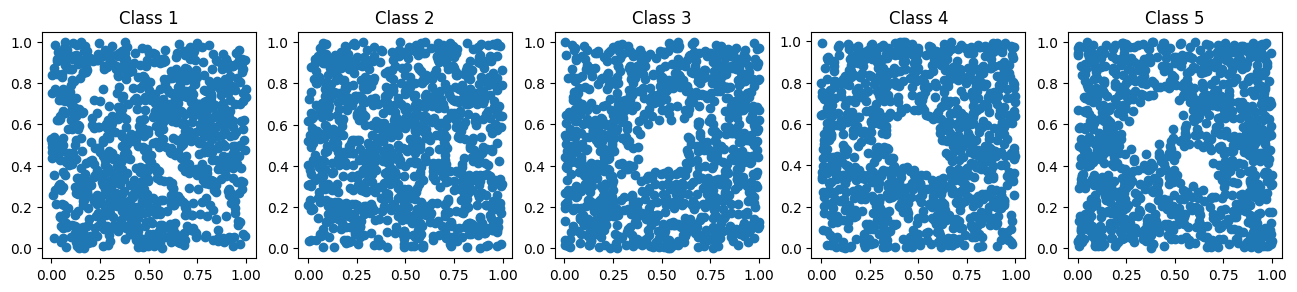}}
\end{figure}

\subsection{Implementation details for other methods}
\label{appendix:impl_details_other}

Table \ref{tab:accs_all_methods} compares the accuracy of SRN (Section~\ref{SRN}) on the ORBITS5K dataset with that of other methods, as reported in the article introducing Perslay \citep{carriere2020perslay}. PSS-K, PWG-K, SW-K, PF-K stand respectively for Persistence Scale Space Kernel \citep{reininghaus2015stable}, Persistence Weighted Gaussian Kernel \citep{kusano2016persistence}, Sliced Wasserstein Kernel \citep{carriere2017sliced} and Persistence Fisher Kernel \citep{le2018persistence}.

In order to test the robustness of Perslay in PyTorch we reimplement it with the following DeepSet \citep{zaheer2017deep} architecture, similar to the architecture described in \cite{carriere2020perslay}: each point in the persistence diagram is upsampled from 2 dimensions to 25 through 2 layers. The point embeddings are further aggregated with a top 5 aggregation for each dimension, which is composed with a final linear layer.

\begin{table}[hbtp]
\floatconts
    {tab:accs_all_methods}
    {\caption{Accuracy of our method and of other methods on the ORBITS5K dataset as reported in \cite{carriere2020perslay}. See details in Appendix \ref{appendix:impl_details_other}.}}
    {\begin{tabular}{lllllll}
    \toprule
         & \bfseries SRN ($H_1$ only) & \bfseries PSS-K & \bfseries PWG-K & \bfseries SW-K & \bfseries PF-K & \bfseries PersLay\\
         \midrule
         Accuracy & $79.6 (\pm 0.3)$ & $72.38 (\pm 0.7)$ & $76.63 (\pm 0.9)$ & $83.6 (\pm 0.9) $ & $85.9 (\pm 0.8)$ & $87.7 (\pm 1.0)$ \\
         \bottomrule
    \end{tabular}}
\end{table}

\subsection{Adversarial attacks}
\label{appendix:adv_attacks}

\cite{carlini2017towards} describe a method to find adversarial examples for data in $L_p$ spaces. We adapt the method to data in the space of persistence diagrams endowed with the $p$-Wasserstein distance $W_p$, with $p\in [1,\infty]$. The optimization objective becomes:
\[
\min_{D'} W_p(D,D') - \lambda \mathcal{L}(g(D')),
\]
where $\lambda$ is a hyperparameter and $\mathcal{L}$ is the loss of the neural network (i.e.\ cross-entropy). The Wasserstein distance between two persistence diagrams is a differentiable function of the coordinates of the points in the persistence diagrams \citep{carriere2021optimizing}. We initialize $D'$ as $D$, to which a fixed number of persistence diagram points may be added and initialized close to the diagonal $\Delta$. The gradient w.r.t.\ $D'$ is computed using PyTorch and the \textit{PyTorch Topological} package and Projected Gradient Descent iterations are taken, where the data is projected to the feasible area of the persistence diagram (i.e.\ the set of points $(a,b)\in \mathbb{R}$ with $0\le a\le b$).

\subsection{Implementation details for Stable Rank Network}
\label{appendix:impl_details_sr_linf}

\begin{figure}[htbp]
\floatconts
  {fig:figure_pipeline}
  {\caption{A Persistent Homology Machine Learning pipeline using Stable Rank Network.}}
  {\includegraphics[width=1\linewidth]{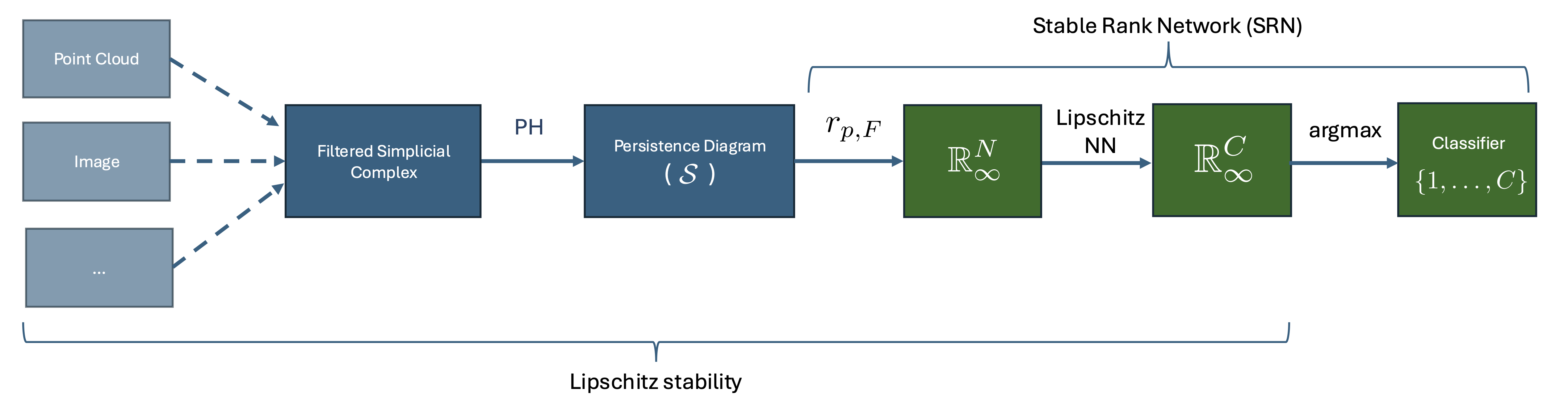}}
\end{figure}

To simplify the analysis in terms of Wasserstein distance we choose to work with only one persistence diagram per sample, i.e.\ choosing one homology degree. We choose $H_1$ as it is more distinctive.

Following the discussion in Appendix \ref{appendix:stab_rank_stability}, $F$ should be a learnable function (i.e.\  a parametrization of $F$ with parameters that can be optimized when training SRN is needed) and  $\sup{F'}$ should be easily computable. In our experiments we choose $f$ to be a Gaussian mixture and as in Appendix \ref{appendix:stab_rank_stability}, $F(t) =\int_0^t f\, d\lambda$, for all $t\in \mathbb{R}_{\ge 0}$. $f$ has learnable weights, means and standard deviations of the Gaussian components. For the ORBIT5K dataset, similar classification performance was achieved by choosing $F$ to be the identity function, which is what is reported in Table \ref{tab:rob_accs}, and Table \ref{tab:accs_all_methods}.

For the Lipschitz neural network (Section \ref{subsec:LipNN}) we use 5 layers with sizes $(1200, 700, 300, 80, 5)$. After each layer, centering was applied (i.e.\ batch normalization but where only subtraction by mean feature-wise is applied, not division by standard deviation feature-wise). The weights of the first layer were initialized uniformly at random but scaled proportionally to the standard deviation of features in the dataset.

\begin{figure}[htbp]
\floatconts
  {fig:dist_distances}
  {\caption{Distribution of certified $\epsilon$-robustness for correctly classified samples of the test set, for the different classes.}}
  {\includegraphics[width=1\linewidth]{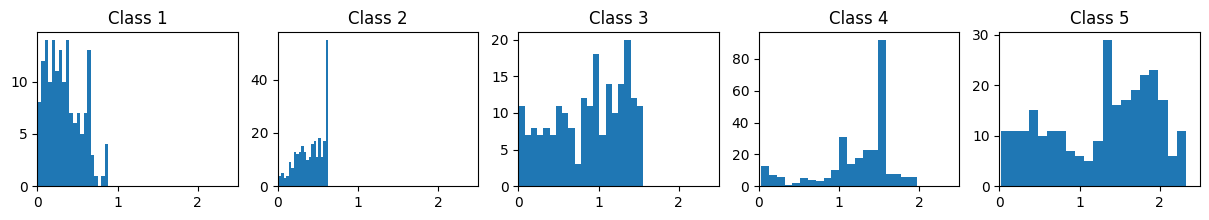}}
\end{figure}

\begin{table}[hbtp]
\floatconts
    {tab:avgdisth1}
    {\caption{Average bottleneck distances between $H_1$ PDs for samples of different classes in the dataset.}}
    {\begin{tabular}{l|lllll}
    \toprule
         Class & \bfseries 1 & \bfseries 2 & \bfseries 3 & \bfseries 4 & \bfseries 5 \\
         \midrule
        \bfseries 1     & 2.383 & 2.719 & 3.533 & 17.334 & 9.011 \\
        \bfseries 2     &         & 0.914 & 4.115 & 19.439 & 11.038 \\
        \bfseries 3     &         &         & 2.536 & 15.492 & 7.750 \\
        \bfseries 4     &         &         &         & 4.077  & 9.535 \\
        \bfseries 5     &         &         &         &         & 5.043 \\
         \bottomrule
    \end{tabular}}
\end{table}

\end{document}